\documentclass[letterpaper, 10 pt, conference]{ieeeconf}  

\IEEEoverridecommandlockouts                              
\usepackage[utf8]{inputenc}
\usepackage[T1]{fontenc}
\usepackage{xcolor}
\usepackage{graphicx}
\usepackage{xspace}

\usepackage{algorithm}
\usepackage{algpseudocode}
\algrenewcommand\algorithmicrequire{\textbf{Input:}}
\newcommand{\CommentState}[1]{\Statex\hspace{\algorithmicindent}{\color{blue}// #1}}

\usepackage{amsmath,amsfonts,amssymb}
\usepackage{bm,fixmath} 
\newtheorem{assumption}{Assumption}
\newtheorem{lemma}{Lemma}

\newtheorem{theorem}{Theorem}

\allowdisplaybreaks 

\newcommand{\alg}{LT-ADMM-CC\xspace}
\newcommand{\tgrad}{t_g\xspace}
\newcommand{\tcomm}{t_c\xspace}
\DeclareMathOperator*{\argmin}{arg\,min}

\newcommand{\N}{\mathbb{N}}
\newcommand{\R}{\mathbb{R}}

\newcommand{\norm}[1]{\left\lVert#1\right\rVert}

\title{\LARGE \bf
Jointly Computation- and Communication-Efficient Distributed Learning
}
\author{%
	Xiaoxing~Ren$^{1}$, Nicola~Bastianello$^{2\star}$, Karl~H.~Johansson$^{2}$, Thomas~Parisini$^{1,3,4}$
	\thanks{The work of X. R and T. P. was partially supported by European Union's Horizon 2020 research and innovation programme under grant agreement no. 739551 (KIOS CoE).}
\thanks{The work of N.B. and K.H.J. was partially supported by the European Union’s Horizon Research and Innovation Actions programme under grant agreement No. 101070162, and partially by Swedish Research Council Distinguished Professor Grant 2017-01078 Knut and Alice Wallenberg Foundation Wallenberg Scholar Grant.}
	\thanks{$^{1}$Department of Electrical and Electronic Engineering, Imperial College London, London, United Kingdom.}%
	\thanks{$^{2}$School of Electrical Engineering and Computer Science, and Digital Futures, KTH Royal Institute of Technology, Stockholm, Sweden.}
	\thanks{$^{3}$Department of Electronic Systems, Aalborg University, Denmark.}%
	\thanks{$^{4}$Department of Engineering and Architecture, University of Trieste, Trieste, Italy.}%
	\thanks{$^{\star}$Corresponding author. Email: {\tt\footnotesize nicolba@kth.se}.}%
}

\begin{document}

\maketitle

\thispagestyle{empty}
\pagestyle{empty}


\begin{abstract}
We address distributed learning problems over undirected networks. Specifically, we focus on designing a novel ADMM-based algorithm that is jointly computation- and communication-efficient.
Our design guarantees computational efficiency by allowing agents to use stochastic gradients during local training.
Moreover, communication efficiency is achieved as follows: i) the agents perform multiple training epochs between communication rounds, and ii) compressed transmissions are used.
%
%
We prove \textit{exact} linear convergence of the algorithm in the strongly convex setting.
We corroborate our theoretical results by numerical comparisons with state of the art techniques on a classification task.
\end{abstract}

\section{Introduction}\label{sec:intro}
Smart devices equipped with computational and communications resources underwent, in recent years, a widespread adoption in many applications, including power grids, robotics, traffic and sensor networks \cite{molzahn_survey_2017,nedic_distributed_2018}.
These devices then become the components of multi-agent systems that can collect data and cooperatively achieve learning tasks. However, their resources (CPU and communications) might be limited. Thus, in this paper we focus on the design of a \textit{distributed learning algorithm that is jointly computation- and communication-efficient}.

Different techniques have been explored to design \textit{computation-efficient} algorithms. The main solution is the use of stochastic gradients, which allow the agents to update their models using only a subset of their local data \cite{koloskova2019decentralized,alghunaim_local_2023}.
However, stochastic gradients might cause inexact convergence, and \textit{variance reduction} was proposed to solve this issue, see \textit{e.g.}, \cite{li_variance_2022,jiang_distributed_2023}.

In terms of \textit{communication-efficiency}, the main methods adopted in machine learning are \textit{compression} and \textit{local training}.
Compression allows to reduce the size of communications exchanged by the agents, for example by transmitting only the most significant components of a model.
Different distributed algorithms were designed to use compressed communications (see, for instance 
\cite{yi2022communication,zhang2023innovation,huang2024cedas,liulinear}.
Of these, \cite{yi2022communication,zhang2023innovation} also implement an \textit{error feedback} mechanism, which ensures exact convergence even when using compression.

While compression reduces the size of communications, local training instead reduces their frequency.
The idea is to allow the agents to perform multiple steps of training between each round of communications.
This paradigm has been applied \textit{e.g.} in \cite{hien_nguyen_performance_2023,liu_decentralized_2023} to gradient tracking, in \cite{alghunaim_local_2023} to the distributed dual ascent method, and in our works \cite{ren2024distributed,ren2025communication} that focus on redesigning the distributed ADMM.

It is worth noting that the algorithms we reviewed integrate only one or two computation- or communication-efficient design strategies, thus resulting in inexact convergence when dealing with both stochastic gradients and compression.
Therefore, in this paper we devise an algorithm aiming at joint computational and communication efficiency. Our main contributions are:
\begin{itemize}
    \item We propose a novel algorithm, \alg (Local Training ADMM with Compressed Communication), based on ~\cite{ren2025communication}, which employs stochastic gradients with variance reduction for computation-efficiency, and compression and local training for communication-efficiency. Additionally, our design integrates error feedback.

    \item We analyze the convergence of \alg in a strongly convex setting, and prove its exact linear convergence to the optimal solution. This important result is enabled by double feedback loop of variance reduction and error feedback, which asymptotically reject the stochastic gradient and compression errors, respectively.
    
    \item We provide a numerical comparison of \alg with state-of-the-art alternatives in the context of a classification task, thus highlighting its exact convergence and showcasing its superior performance.
\end{itemize}

\section{Algorithm Design and Analysis}

We start by formally introducing the problem, then present the proposed algorithm and analyze its convergence.

\subsection{Problem Formulation}
Consider a network $\mathcal{G} = (\mathcal{V},\mathcal{E})$ consisting of $N$ agents, each of which has access to a local dataset that defines the cost
\begin{equation}\label{eq:erm-cost}
    f_{i}({x}) = \frac{1}{m_i} \sum_{h=1}^{m_i} f_{i,h}(x) \, ,
\end{equation}
with $f_{i,h}: \mathbb{R}^{n} \rightarrow \mathbb{R}$ being the loss function associated to data point $h \in \{ 1, \ldots, m_i \}$.
The goal is for the network to solve the following consensus optimization problem
\begin{equation}\label{eq:optimization-problem}
    \min_{x_i \in\mathbb{R}^{n}, \ i \in \mathcal{V}} \ \ \frac{1}{N} \sum_{i=1}^{N}f_{i}({x}_i) \quad \text{s.t.} \ \ x_1 = x_2 = \cdots = x_N \, ,
\end{equation}
where the objective sums the local costs~\eqref{eq:erm-cost}, and the constraints enforce agreement on a shared trained model.
In the following, we denote the (unique) optimal solution as $\mathbf{X}^* = \mathbf{1}_N \otimes x^*$, where $\otimes$ denotes Kronecker product, and $x^* = \argmin_{x \in \R^n} \sum_{i = 1}^N f_i(x)$.

We characterize the problem via the following assumptions.

\begin{assumption}\label{as:local-costs}
The cost $f_i$ of each agent $i \in \mathcal{V}$ is $L$-smooth and $\mu$-strongly convex, with $0 < \mu \leq L < \infty$.%
\footnote{
We recall that a function $f : \R^n \to \R$ is $L$-smooth if it is differentiable and $\Vert \nabla f_{i}(x)-\nabla f_{i}(y)\Vert  \leq L \Vert x-y \Vert$, $\forall x, y \in \mathbb{R}^n$; moreover, $f$ is $\mu$-strongly convex if $\left\langle  \nabla f_{i}(x)-\nabla f_{i}(y), x-y\right\rangle  \geq \mu \Vert x-y \Vert^2$, $\forall x, y \in \mathbb{R}^n$.
}
\end{assumption}

\begin{assumption}\label{as:graph}
$\mathit{\mathcal{G}} = (\mathcal{V},\mathcal{E})$ is connected and undirected.
\end{assumption}

These assumptions are standard for distributed problems, with only strong convexity being somewhat restrictive.
However, strong convexity is instrumental to prove linear convergence; and we remark that it can be relaxed to allow for nonconvex problems extending the analysis of \cite{ren2025communication}.

\subsection{Algorithm design}
We recall that we focus on designing a distributed learning algorithm that is jointly computation- and communication-efficient.
To this end, we start from the distributed ADMM of \cite{bastianello_asynchronous_2021}, and in the following steps we integrate some suitable design modifications.
Therefore, our starting point is the algorithm characterized by the updates:
\begin{subequations}\label{eq:admm}
\begin{align}
    x_{i, k+1} &= \operatorname{prox}_{f_i}^{1 / \rho\left|\mathcal{N}_i\right|} \left( \sum\nolimits_{j \in \mathcal{N}_i} z_{i j, k} / \rho \left|\mathcal{N}_i\right| \right) \label{eq:admm-x} \\
    z_{i j, k+1} &= 0.5 \left( z_{i j, k} - z_{j i, k} + 2 \rho x_{j, k+1}\right) \label{eq:admm-z}
\end{align}
\end{subequations}
where $\rho>0$ is a penalty parameter, $\operatorname{prox}_{f_i}^{1 / \rho\left|\mathcal{N}_i\right|}(z) = \underset{x \in \mathbb{R}^n}{\arg \min }\{f_i(x) + (\rho\left|\mathcal{N}_i\right|/2) \|x-z\|^2 \}$, and $z_{ij,k}$ and $z_{ji,k}$ are edge-wise auxiliary variables.

\paragraph{Communication-efficiency 1}
The distributed implementation of~\eqref{eq:admm} entails agent $j \in \mathcal{N}_i$ sending $- z_{ji,k} + 2\rho x_{j,k+1}$ to $i$ \cite{bastianello_asynchronous_2021}. However, in a learning setting where $n \gg 1$, this requires prohibitively large bandwidth, and as a solution compressed communications can be employed.
Letting $\mathcal{C} : \R^n \to \R^n$ be a compression operator,%
\footnote{Precise assumptions will be introduced in section~\ref{subsec:convergence}.}
we could then allow agent $j$ to transmit $\mathcal{C} \left( - z_{ji,k} + 2\rho x_{j,k+1} \right)$. This design choice would reduce the communication burden, but would also result in inexact convergence.
Thus, we also integrate an \textit{error feedback} mechanism that asymptotically reduces to zero the error induced by compression.
In particular, we define the auxiliary variables $u_{i,k}$ and $s_{ij,k}$, and rewrite~\eqref{eq:admm-z} as:
\begin{equation} \label{eq:admm_z_compress}
    z_{i j, k+1} = 0.5 \left(\hat{z}_{i j, k} - \hat{z}_{j i, k} \right)+  r \rho x_{i, k+1} 
 - r \rho ( \hat{x}_{i, k+1}  - \hat{x}_{j, k+1} )
\end{equation}
where $r > 0$ and we define
\begin{subequations}
\begin{align} 
    \hat{{x}}_{i,k+1} &= u_{i, k+1} + \mathcal{C}( {x}_{i,k+1} -  {u}_{i, k+1} ), \label{eq:compress_x_local} \\
    \hat{{z}}_{ij,k+1} &= s_{ij, k+1} + \mathcal{C}( {z}_{ij,k+1} -  {s}_{ij, k+1} ) \label{eq:compress_z_local}
\end{align}
\end{subequations}
and, setting $\eta \in (0, 1]$,
\begin{equation} \label{eq:compress_auxiluary}
\begin{aligned}
    {u}_{i, k+1}  &=  (1- \eta){u}_{i, k} + \eta  \hat{x}_{i,k} \\
    s_{ij,k+1} &= \hat{z}_{ij,k}.
\end{aligned}
\end{equation}
%


It can be seen that the transmission of $\hat{{x}}_{j,k+1}$ involves the exact transmission of $u_{j,k+1}$ because of \eqref{eq:compress_x_local}. To overcome this, we let agent $i$ keep a copy $\tilde{u}_{i,k}$ of $u_{j,k}$, according to \eqref{eq:compress_auxiluary}, set $\tilde{u}_{i,0}=u_{j,0}$,  agent $i$ can maintain $\tilde{u}_{i,k+1} = {u}_{j,k+1}$ by only receiving $\mathcal{C}( {x}_{j,k} -  {u}_{j, k} )$ by mathematical induction. Specifically, given $\tilde{u}_{i,k} ={u}_{j,k} $, $\tilde{u}_{i,k+1} = (1-\eta)\tilde{u}_{i,k} + \eta ( \tilde{u}_{i, k} + \mathcal{C}( {x}_{j,k} - {u}_{j, k} ) )={u}_{j,k+1} $, $\forall k \geq 0$. 
The same holds for $\hat{\mathbf{z}}_{ji,k}$, we let agent $i$ keep a copy $\tilde{s}_{ij,k}$ of $s_{ji,k}$.

\paragraph{Communication efficiency 2}
Besides reducing the required bandwidth via compression, in our design we also reduce the frequency of communications via local training.
Following the idea in \cite{ren2025communication}, we notice that update~\eqref{eq:admm-x} requires the solution of an optimization problem, which in general lacks a closed form.
The idea then is to allow the agents to approximate its solution with a finite number $\tau$ of gradient-based steps:
\begin{equation}\label{eq:local-training}
\begin{aligned}
 &\phi_{i,k}^{0} =x_{i, k} \\
&\phi_{i,k}^{t+1} = \phi_{i, k}^t + \hspace{3cm} t =0, \ldots, \tau-1 \\ &\ -\gamma g_i(\phi_{i,k}^{t}) -\beta \rho\left|\mathcal{N}_i\right| \Bigg(r^2 x_{i,k}-\frac{r}{\rho\left|\mathcal{N}_i\right|}\sum_{j \in \mathcal{N}_i} z_{i j, k} \Bigg)
\\&
x_{i, k+1}  =\phi_{i,k}^{\tau}
\end{aligned}
\end{equation}
where $g_i(\phi_{i,k}^{t}) = \nabla f_i(\phi_{i,k}^{t})$, $\gamma > 0$ is the step-size and $\beta > 0$ an additional regularization weight.
As a consequence, one round of communication is performed every $\tau$ local updates.

\paragraph{Computational efficiency}
The use of full gradient evaluations $g_i(\phi_{i,k}^{t}) = \nabla f_i(\phi_{i,k}^{t})$ might be computationally expensive when $m_i \gg 1$ in~\eqref{eq:erm-cost}.
Thus, we allow the agents to use a variance reduced stochastic gradient estimator \cite{defazio2014saga}, where
each agent maintains a table of component gradients $\{ \nabla f_{i,h}(r_{i, h, k}^t) \}, h = 1, \ldots,m_i$, where $r_{i, h, k}^t$ represents the most recent iterate at which the component gradient was computed. This table is reset at the start of each new local training, and the agents estimate their local gradients as
\begin{equation}\label{eq:saga-gradient}
\begin{split}
g_i\left(\phi_{i,k}^t\right) &= \frac{1}{|\mathcal{B}_i|} \sum_{h \in \mathcal{B}_i} \left( \nabla f_{i, h}\left( \phi_{i, k}^t \right) - \nabla f_{i, h}\left(r_{i, h, k}^{t} \right) \right) \\ &+\frac{1}{m_i} \sum_{h=1}^{m_i} \nabla f_{i, h}(r_{i, h, k}^{t}).
\end{split}
\end{equation}
where $\mathcal{B}_i$ represents a randomly selected subset of indices from $\{ 1, \ldots, m_i\}$, with $|\mathcal{B}_i| < m_i$.
The estimated gradient is then used to update $\phi_{i,k}^{t+1}$ according to~\eqref{eq:local-training}; after each update, the agents refresh their local memory by setting $r_{i,h,k}^{t+1} = \phi_{i,k}^{t+1}$ if $h \in \mathcal{B}_i$, and $r_{i,h,k}^{t+1} = r_{i,h,k}^{t}$ otherwise.
This update requires a full gradient computation at the beginning of each local training step, in the following steps ($t > 0$), each agent only computes $|\mathcal{B}_i|$ component gradients.

\smallskip

The result of our design is reported in Algorithm~\ref{alg:lt-saga-admm}.
\begin{algorithm}[!ht]
\caption{}
\label{alg:lt-saga-admm}
\begin{algorithmic}[1]
\Require For each node $i$, initialize $x_{i,0}= z_{ij, 0}$, $u_{i,0} = \tilde{u}_{i,0} =0$, $s_{ij,0} = \tilde{s}_{ij,0}= 0$, $j \in \mathcal{N}_i$. Set the penalty parameter $\rho>0$, the number of local training steps $\tau >0$, and the parameters $\gamma, \beta, r>0$, $0<\eta \leq 1$.
 
	\For{$k = 0,1,\ldots$ every agent $i$}
 \CommentState{local training}

    \State $\phi_{i,k}^0 = x_{i,k}$, {$r_{i, h, k}^{0}  = x_{i, k}$, for all $h \in \{ 1, \ldots. m_i \}$}
    
    \For{$t = 0, 1, \ldots, \tau-1$}

\State Draw the batch $\mathcal{B}_i$ uniformly at random

\State Update the gradient estimator according to \eqref{eq:saga-gradient}

\State Update $\phi_{i,k}$ according to \eqref{eq:local-training}

\State {If $h \in \mathcal{B}_i$ update $r_{i,h,k}^{t+1} = \phi_{i,k}^{t+1}$, else $r_{i,h,k}^{t+1} = r_{i,h,k}^{t}$}
	
    \EndFor
    
\State Set $x_{i,k+1} = \phi_{i,k}^\tau$, update $u_{i,k+1} $ and $s_{ij,k+1}$ according to \eqref{eq:compress_auxiluary}, update $\hat{{x}}_{i,k+1}$ according to \eqref{eq:compress_x_local}
    
	\CommentState{communication}
    \State Transmit $\mathcal{C}( {z}_{ij,k} -  {s}_{ij, k} )$ and $\mathcal{C}( {x}_{i,k+1} -  {u}_{i, k+1} )$ to each neighbor $j \in \mathcal{N}_i$, and receive the corresponding transmissions
    
\CommentState{auxiliary update}
\State Update the local copy $\tilde{u}_{i,k+1}$ and $\tilde{s}_{ij,k+1}$, update $z_{ij,k+1}$ according to \eqref{eq:admm_z_compress}.
\State Update $ \hat{{z}}_{ij,k+1}$ according to \eqref{eq:compress_z_local}
	\EndFor
\end{algorithmic}
\end{algorithm}

\subsection{Convergence analysis}\label{subsec:convergence}
We start by introducing suitable assumptions on the compressor operator.

\begin{assumption}\label{as:compressor}
The compression operator $\mathcal{C}: \mathbb{R}^n \rightarrow \mathbb{R}^n$ is unbiased 
$ \mathbb{E}\bigl[\mathcal{C}(x)\bigr] = x $,
and there exists a constant $p > 1$ such that 
$\mathbb{E}\bigl[\|\mathcal{C}(x) \|^2\bigr] \le p\|x\|^2, \forall x \in \mathbb{R}^n.$
\end{assumption}

\begin{assumption}\label{as:independent compressor}
We assume that all agents' compressors are mutually independent among the agents, \textit{i.e.}, their outputs are mutually independent random variables.
\end{assumption}

\smallskip

We can now characterize the convergence of Algorithm~\ref{alg:lt-saga-admm}. The result is proved in the Appendix.

\begin{theorem}\label{th:convergence}
Let Assumptions~\ref{as:local-costs},~\ref{as:graph},~\ref{as:compressor}~and~\ref{as:independent compressor} hold.
Let   $\left\lbrace \mathbf{X}_{k} \right\rbrace_{k \in \N}$ be the trajectory generated by \alg.
Then with sufficiently small $\gamma$, bounded $p$, there exist positive parameters  $\beta, \tau, r, \rho$ such that the states $\mathbf{X}_k$ converge linearly to the optimal solution $\mathbf{X}^*$.
\end{theorem}

\section{Numerical Results}\label{sec:numerical}
In this section, we compare \alg with state-of-the-art alternatives for a classification task characterized by
\begin{equation}\label{eq:logistic-cost}
    f_i(x) = \sum\nolimits_{h = 1}^{m_i} \log\left( 1 + \exp(- b_i^h a_i^h x) \right) + (\epsilon / 2) \norm{x}^2
\end{equation}
with $a_i^h \in \R^n$, and $b_i^h \in \{ -1, 1 \}$ randomly generated.
We choose a ring network with $N = 10$, set $n = 5$ and $m_i = 100$, and use stochastic gradients with a batch of $|\mathcal{B}| = 1$.

\subsection{\alg performance}
We start by evaluating the performance of \alg with the following unbiased compressors. The other parameters are always set as $\tau = 5$, $\rho = 0.1$, $\beta = 0.2$, $\gamma = 0.3$, $r =1$.

\textit{$b$-bit quantizer}
The first compressor is defined as
$$
\mathcal{C}_1(x) = 
\frac{ \|x\|_{\infty}\,\mathrm{sign}(x) }{2^{b-1}}
\circ
\left\lfloor\;
2^{b-1}\,\frac{\lvert x\rvert}{\|x\|_{\infty}}
\;+\;\kappa
\right\rfloor
$$
where $\mathrm{sign}(x)$, $\lvert\,\cdot\,\rvert$, $\lfloor\,\cdot\,\rfloor$ are applied element-wise, $\circ$ is the element-wise product; and $\kappa \sim \mathcal{U}[0, 1]^n$ is a random perturbation.

\textit{Rand-$k$}
The second compressor is defined as
$$
    \mathcal{C}_2(x) \;=\; \frac{n}{k} \sum\nolimits_{i \in S} x_i \, e_i,
$$
where $S \subset [n]$ is  a subset of cardinality $k$ chosen uniformly at random, and $\{ e_1,\ldots,e_d \}$ is the standard basis in $\mathbb{R}^n$.

\smallskip

Figure~\ref{fig:compressor-comparison} shows the evolution of $\norm{\nabla F(\bar{x}_k)}^2$, with $\bar{x}_k = (1 / N) \sum_{i = 1}^N x_{i,k}$ over the iterations $k \in \N$.
\begin{figure}[!ht]
\centering
\includegraphics[scale=0.4]{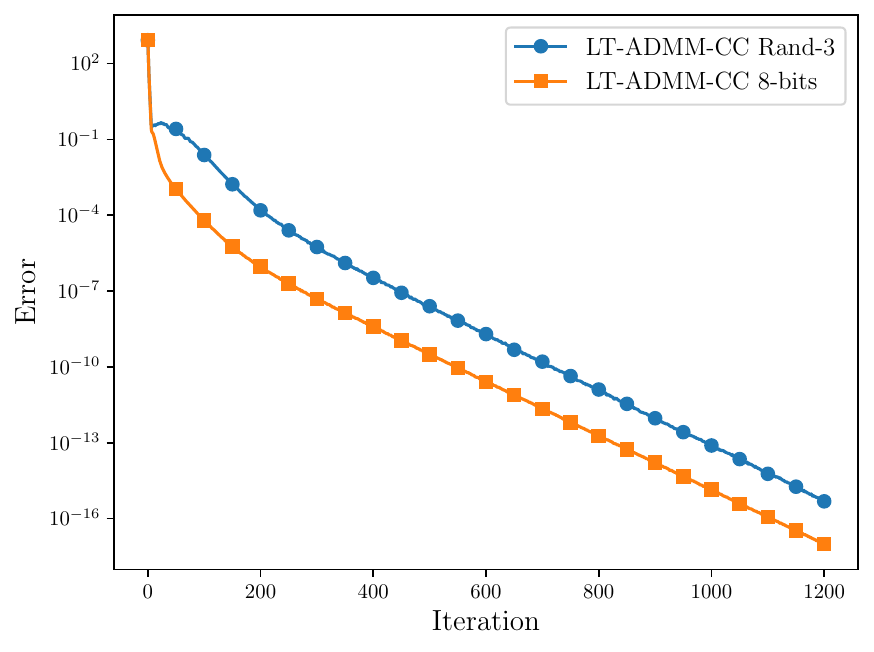}
\caption{ \alg with different compressors.}
\label{fig:compressor-comparison}
\end{figure}
As predicted by the theory, the algorithm achieves exact convergence for both $\mathcal{C}_1$ and $\mathcal{C}_2$, although the specific compressor might affect the speed of convergence.

\subsection{Comparison}
We compare now the performance of \alg with alternative distributed optimization methods that employ compression: CEDAS \cite{huang2024cedas}, COLD \cite{zhang2023innovation}, DPDC \cite[Algorithm 1]{yi2022communication}, and LEAD \cite{liulinear}.
For all algorithms, we use an $8$-bit quantizer and stochastic gradients with batch-size $|\mathcal{B}| = 1$; we also hand-tune the parameters of the algorithms to achieve optimal performance.

These algorithms have different computational and communication complexities. To account for this in our simulations, we assign a time cost of $\tgrad$ for a component gradient evaluation ($\nabla f_{i,h}$), and of $\tcomm$ for a round of communications. The total time-cost incurred by each algorithm is then reported in Table~\ref{tab:time-comparison}.
\begin{table}[!ht]
\centering
\caption{Computation time of the algorithms over $\tau$ iterations.}
\label{tab:time-comparison}
\begin{tabular}{cc}
    \hline
    Algorithm [Ref.] & Time \\
    \hline
      LEAD \cite{liulinear} 
      & $\tau (\tgrad +  \tcomm)$ \\
         CEDAS \cite{huang2024cedas}
      & $\tau (\tgrad + 2 \tcomm)$  \\
  COLD \cite{zhang2023innovation}\&  DPDC \cite{yi2022communication}  &  $\tau \left( \tgrad + \tcomm \right)$ stochastic gradient\\
   COLD \cite{zhang2023innovation}\&  DPDC \cite{yi2022communication}  &  $\tau \left(m_i \tgrad + \tcomm \right)$ full gradient \\
    \hline
    \alg & $(m_i + \tau - 1) \tgrad + 2\tcomm$  \\
    \hline
\end{tabular}
\end{table}

Figure~\ref{fig:error-comparison} reports the evolution of $\norm{\nabla F(\bar{x}_k)}^2$ against time, with $\tcomm = 10 \tgrad$ to represent a scenario where communication is expensive.
\begin{figure}[!ht]
    \centering
    \includegraphics[scale=0.4]{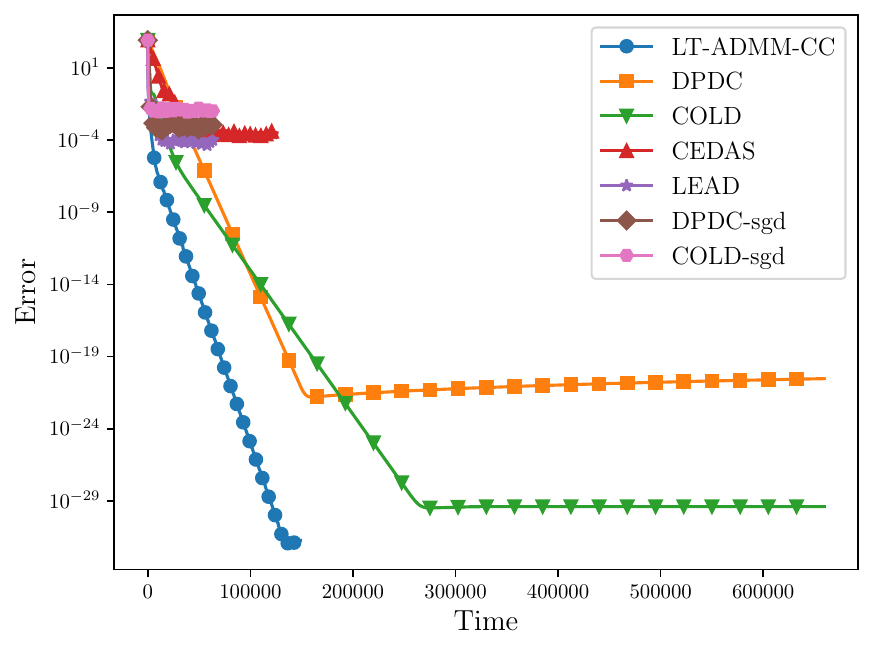}
    \caption{Comparison of distributed optimization algorithms with compressed communication.}
    \label{fig:error-comparison}
\end{figure}
We notice that LEAD, CEDAS, DPDC-sgd, COLD-sgd only converge to a neighborhood of the optimal solution. This is due to the fact that they employ stochastic gradients, but no variance reduction.
\alg instead converges exactly owing to variance reduction and error feedback.
We remark that DPDC and COLD also employ error feedback, hence they converge exactly when using full gradients. Notice, however, that using full gradients entails a higher time-cost, as demonstrated by their slower convergence as compared to \alg.

\section{Concluding Remarks}
In this paper, we proposed a novel distributed learning algorithm that is jointly computation- and communication-efficient. This is achieved by integrating communication compression, local training, and variance reduction.
We proved the exact, linear convergence of the algorithm, and compared it numerically with state-of-the-art alternatives.
Future research efforts will be devoted to weakening some assumptions, for example the strong convexity one, and to apply the proposed algorithm to real-world scenarions.

\appendix

In the following, we report the proof of Theorem~\ref{th:convergence}.

\subsection{Preliminary transformation}
Denote  $ \mathrm{\Phi}_k^t = \operatorname{col}\{\phi^t_{1,k}, \phi_{2,k}^t, ..., \phi_{N,k}^t\}
$, $ G(\mathrm{\Phi}_k^t) = \operatorname{col}\{ g_1(\phi_{1,k}^t), g_2(\phi_{2,k}^t),..., g_N(\phi_{N,k}^t)\}
$, $ \mathrm{F}(\mathbf{X}) = \operatorname{col}\{  f_1(x_1),   f_2(x_2),...,  f_N(x_N)\}$, $\mathbf{Z} = \operatorname{col}\{z_{ij}\}_{i,j \in \mathcal{E}} 
$.
Define 
$
\mathbf{A}= \operatorname{blk\,diag}\{ \mathbf{1}_{d_i} \}_{i \in \mathcal{V}} \otimes \mathbf{I}_n \in \mathbb{R}^{Mn \times Nn},
$
where $d_i = |\mathcal{N}_i|$ is the degree of node $i$, and $M = \sum_i |\mathcal{N}_i|$.
$\mathbf{P} \in \mathbb{R}^{Mn \times Mn}$ is a permutation matrix that swaps $e_{ij}$  with $e_{ji}$. 
$\mathbf{A}^T\mathbf{P}\mathbf{A} = \Tilde{\mathbf{A}}$ is the adjacency matrix, $  \mathbf{A}^T\mathbf{A} = \operatorname{diag}\{ d_i \mathbf{I}_n \}_{i \in \mathcal{V}}$ is the degree matrix, denote $d_u$ as the largest degree among the agents. Denote the largest and smallest nonzero eigenvalue of $\mathbf{L} = \mathbf{D} - \Tilde{\mathbf{A}}$ as  $\lambda_u$ and $\lambda_l$, respectively.

The compact form of \alg is:
\begin{subequations}\label{eq:compact-admm}
\begin{equation} \label{eq:compact-compress}
\begin{aligned}
&\hat{\mathbf{X}}_{k} = \mathbf{U}_k + \mathcal{C}( \mathbf{X}_k -  \mathbf{U}_k ), \quad
\mathbf{U}_{k+1}  =  (1- \eta) \mathbf{U}_k + \eta  \hat{\mathbf{X}}_{k}, 
\\&\hat{\mathbf{Z}}_{k} = \mathbf{S}_k + \mathcal{C}( \mathbf{Z}_k -\mathbf{S}_k ), \qquad
\mathbf{S}_{k+1}  = \hat{\mathbf{Z}}_{k} 
\end{aligned}
\end{equation}
\begin{equation}
\mathbf{X}_{k+1} = \mathbf{X}_{k} - \gamma \sum_{t=0}^{\tau -1}\nabla G(\mathrm{\Phi}_k^t) - \beta r \sum_{t=0}^{\tau -1}(r\rho \mathbf{A}^T\mathbf{A}\mathbf{X}_k - \mathbf{A}^T \mathbf{Z}_k) \label{eq:compact-admm-x}
\end{equation}
\begin{equation}
\mathbf{Z}_{k+1} =  \frac{1}{2} (\mathbf{I} - \mathbf{P})\hat{\mathbf{Z}}_{k} + r \rho \mathbf{A}\mathbf{X}_{k+1} - r \rho (\mathbf{I} - \mathbf{P}) \mathbf{A} \hat{\mathbf{X}}_{k+1}. \label{eq:compact-admm-z}
\end{equation}
\end{subequations}

We introduce the following variables
\begin{equation} \label{y_tilde_y}
\begin{aligned}
&\mathbf{Y}_k= r \mathbf{A}^T \mathbf{Z}_{k} - \frac{\gamma}{\beta}\nabla F(\mathbf{\bar{X}}_k) -r^2\rho \mathbf{D} \mathbf{X}_k
\\& \Tilde{\mathbf{Y}}_k= r\mathbf{A}^T \mathbf{P}\mathbf{Z}_{k} + \frac{\gamma}{\beta} \nabla \mathrm{F}(\bar{\mathbf{X}}_k) - r^2\rho \mathbf{D} \mathbf{X}_k,
\end{aligned}
\end{equation}
where $\bar{\mathbf{X}}_k = \mathbf{1}_N \otimes \bar{x}_k$, with $\bar{x}_k = \frac{1}{N} \mathbf{1}^T \mathbf{X}_k$.
Multiplying both sides of \eqref{eq:compact-admm-z} by $r\mathbf{1}^T\mathbf{A}^T$, and using the initial condition, we obtain $r\mathbf{1}^T\mathbf{A}^T\mathbf{Z}_{k+1} = r^2\rho \mathbf{1}^T \mathbf{D} \mathbf{X}_{k+1}$ for all $k\geq 0$.
As a consequence $\bar{\mathbf{Y}}_k = \mathbf{1} \otimes \frac{1}{N}  \mathbf{1}^T\nabla \mathrm{F}(\bar{\mathbf{X}}_k) = \frac{\gamma}{\beta}\mathbf{1} \otimes \frac{1}{N} \sum_{i}\nabla f_i(\bar{x}_k)$, and \eqref{eq:compact-admm} can be rewritten as
\begin{equation}\label{eq:compact-admm-2}
\begin{aligned}
& \begin{bmatrix}
\mathbf{X}_{k+1}\\
\mathbf{Y}_{k+1}\\
\Tilde{\mathbf{Y}}_{k+1}
\end{bmatrix} = \begin{bmatrix}
\mathbf{I}  &  \beta \tau \mathbf{I} & \mathbf{0}  \\
\rho \Tilde{\mathbf{L}}  &  \rho \Tilde{\mathbf{L}}\beta \tau +  \frac{1}{2} \mathbf{I}  & - \frac{1}{2}\mathbf{I}  \\
\mathbf{0}  &   - \frac{1}{2}\mathbf{I} &  \frac{1}{2} \mathbf{I}
\end{bmatrix} \otimes \mathbf{I}_n \begin{bmatrix}
     \mathbf{X}_{k}\\
    \mathbf{Y}_{k}\\
   \Tilde{\mathbf{Y}}_{k}
\end{bmatrix} 
-   \mathbf{h}_k, 
\end{aligned}
\end{equation}
where 
$
\Tilde{\mathbf{L}} = r^2(\tilde{\mathbf{A}}- \mathbf{D})= -r^2 \mathbf{L}
$
and 
$$
\begin{aligned}
\mathbf{h}_k &= 
\Bigg[ \gamma  \sum_{t=0}^{\tau -1}( \nabla G(\mathrm{\Phi}_k^t) -  \nabla \mathrm{F}(\bar{\mathbf{X}}_k) ) ; \\
&\gamma  \rho  \Tilde{\mathbf{L}}\sum_{t=0}^{\tau -1}( \nabla G(\mathrm{\Phi}_k^t) -  \nabla \mathrm{F}(\bar{\mathbf{X}}_k)  ) \\&+
\frac{\gamma}{\beta}(\nabla  \mathrm{F}(\bar{\mathbf{X}}_{k+1}) -  \nabla  \mathrm{F}(\bar{\mathbf{X}}_{k}) )+ \Tilde{\mathbf{L}}(  \mathbf{X}_{k+1} - \hat{\mathbf{X}}_{k+1} )  ; \\& \frac{\gamma}{\beta}(-\nabla  \mathrm{F}(\bar{\mathbf{X}}_{k+1}) +  \nabla  \mathrm{F}(\bar{\mathbf{X}}_{k})  )- \Tilde{\mathbf{L}}(  \mathbf{X}_{k+1} - \hat{\mathbf{X}}_{k+1} )  \\&- \frac{r\mathbf{A}^T(\mathbf{I} - \mathbf{P})  }{2} ( \mathbf{Z}_{k} - \hat{\mathbf{Z}}_{k}) \Bigg].
\end{aligned}$$

Denote 
$
\|\widehat{\mathbf{\Phi}}_k\|^2=\sum_{i=1}^N \sum_{t=0}^{\tau-1} \|\phi_{i, k}^t-\bar{x}_k \|^2 =\sum_{t=0}^{\tau-1}\|\mathbf{\Phi}_{k}^t-\bar{X}_k \|^2 
$, using Assumption~\ref{as:local-costs} we derive that
\begin{align*}
&\| \sum_{t=0}^{\tau -1}( G(\mathrm{\Phi}_k^t) -  \nabla \mathrm{F}(\bar{\mathbf{X}}_k)   )\|^2
\\&
\leq 2\tau L^2 \|\widehat{\mathbf{\Phi}}_k \|^2
 +   2\tau \sum_{t=0}^{\tau -1} \| G(\mathrm{\Phi}_k^t)  - \nabla F(\mathrm{\Phi}_k^t)   \|^2.
\end{align*}

Denote $\overline{G}(\mathrm{\Phi}_k^t)=\frac{1}{N}  \sum_{i=1}^{N} g_i (\phi_{i, k}^t)$, we have
\begin{equation} \label{G(K)}
\begin{aligned}
&\|\sum_{t=0}^{\tau -1} \overline{G}(\mathrm{\Phi}_k^t) \|^2 
= \|\frac{1}{N} \sum_i \sum_t (\nabla f_i\left(\phi_{i, k}^t\right)-\nabla f_i\left(\bar{x}^k\right) 
\\&+\nabla f_i\left(\bar{x}^k\right)  +  g_i\left(\phi_{i, k}^t\right) -   \nabla f_i\left(\phi_{i, k}^t\right) ) \|^2 
\\&  \leq \frac{3 \tau L^2}{N} \|\widehat{\mathrm{\Phi}}^k\|^2  +  \frac{3 \tau}{N} \sum_i \sum_t\|   g_i\left(\phi_{i, k}^t\right) -   \nabla f_i\left(\phi_{i, k}^t\right) \|^2 
\\&+   3 \tau^2 \| \nabla F(\bar{x}_k) \|^2.
\end{aligned}
\end{equation}
We also have
$\|\nabla F(\bar{\mathbf{X}}_{k+1}) - \nabla F(\bar{\mathbf{X}}_{k}) \|^2 =N L^2 \left\| \bar{x}_{k+1} - \bar{x}_{k} \right\|^2 
 = N  L^2 \gamma^2 \|\sum_t \overline{G}(\mathrm{\Phi}_k^t)\|^2,
$
using  Assumption~\ref{as:compressor}, it further holds that:
\begin{equation} \label{eq:h_k_0}
\begin{aligned}
& \|{\mathbf{h}}_{k}\|^2 
 \leq \gamma^2 ( 1+ 3\rho^2 \|\Tilde{\mathbf{L}} \|^2)  ( 2\tau L^2 \|\widehat{\mathbf{\Phi}}_k \|^2
  \\&+ 2\tau \sum_{t=0}^{\tau -1} \| G(\mathrm{\Phi}_k^t)  - \nabla F(\mathrm{\Phi}_k^t)   \|^2 ) 
\\& + 18 L^2 \frac{\gamma^4}{\beta^2} ( \tau L^2\|\widehat{\mathrm{\Phi}}^k\|^2  + \tau \sum_i \sum_t\left\|   g_i\left(\phi_{i, k}^t\right) -   \nabla f_i\left(\phi_{i, k}^t\right)   \right\|^2 \\&+ N \tau^2\|\nabla F(\bar{x}_k)\|^2)
 + 6 r^4 (p-1) \lambda_u^2  \| \mathbf{X}_{k+1} - {\mathbf{U}}_{k+1} \|^2 \\&+ \frac{3r^2d_u}{4}  \| (\mathbf{I} - \mathbf{P})  ( \mathbf{Z}_{k} - \hat{\mathbf{Z}}_{k}) \|^2.
\end{aligned}
\end{equation}

\subsection{Key bounds}
\begin{lemma} \label{lem:devitaion_aver}
Let Assumption~\ref{as:graph} hold, when  $\beta <  \frac{2}{r^2\tau\lambda_u\rho}$,   
 \begin{equation} \label{X_Y_d} 
\|  \bar{\mathbf{X}}_k- \mathbf{X}_k \|^2 \leq \frac{18\beta \tau}{r^2\lambda_l\rho} \|  \widehat{\mathbf{d}}_k\|^2, \quad \|  \bar{\mathbf{Y}}_k- \mathbf{Y}_k \|^2 \leq 9 \|  \widehat{\mathbf{d}}_k \|^2,
\end{equation}
\begin{equation}\label{eq:delta-hat}
\widehat{\mathbf{d}}_{k+1}=\mathbf{\Delta} \widehat{\mathbf{d}}_k- \widehat{\mathbf{h}}_{k},
\end{equation}
where $
\widehat{\mathbf{d}}_k = \widehat{\mathbf{V}}^{-1}
\begin{bmatrix}
\widehat{\mathbf{Q}}^T  \mathbf{X}_{k};\widehat{\mathbf{Q}}^T \mathbf{Y}_{k};\widehat{\mathbf{Q}}^T \Tilde{\mathbf{Y}}_{k}
\end{bmatrix}
$, $\widehat{\mathbf{h}}_{k} =  \widehat{\mathbf{V}}^{-1} \widehat{\mathbf{Q}}^T  \mathbf{h}_k$. $\mathbf{\Delta}$ is a block-diagonal matrix  satisfies $\delta= \|\mathbf{\Delta}\| = 1 - {r^2\lambda_l\rho \tau \beta} /{2}$. 
\end{lemma}
\begin{proof}
Refer to \cite[Lemma 1]{ren2025communication}, we know that if a matrix $\widehat{\mathbf{Q}} \in \mathbf{R}^{N \times (N-1)}$ the matrix satisfying $\widehat{\mathbf{Q}} \widehat{\mathbf{Q}} ^T=\mathbf{I}_N-\frac{1}{N} \mathbf{1 1}{ }^T$,  $\widehat{\mathbf{Q}} ^T \widehat{\mathbf{Q}} =\mathbf{I}_{N-1}$  and $\mathbf{1}^T \widehat{\mathbf{Q}} =0$, $\widehat{\mathbf{Q}} ^T \mathbf{1}=0$.
We diagonalize $\mathbf{D}_i=\mathbf{V}_i \mathbf{\Delta}_i \mathbf{V}_i^{-1}$, where $\mathbf{\Delta}_i$ is the diagonal matrix of $\mathbf{D}_i$'s eigenvalues, and  
\begin{equation}
 \mathbf{V}_i= 
 \begin{bmatrix}
   -\beta \tau& d_{12}& d_{13} \\
  1&   d_{22}& d_{23}  \\
   1 &1 & 1
\end{bmatrix}
\end{equation} with
$d_{12}= -\beta\tau  + ((\beta\tilde{\lambda}_i\rho\tau(\beta\tilde{\lambda}_i\rho\tau + 2))^{0.5})/(\tilde{\lambda}_i\rho)$, $d_{13} = -\beta\tau  - ((\beta\tilde{\lambda}_i\rho\tau(\beta\tilde{\lambda}_i\rho\tau + 2))^{0.5} )/(\tilde{\lambda}_i\rho)$, $d_{22} =\tilde{\lambda}_i\rho d_{12} -1$, $d_{23}= \tilde{\lambda}_i\rho d_{13} -1$,  $\tilde{\lambda}_i<0$ is the nonzero eigenvalue of $\Tilde{\mathbf{L}}$.
We conclude that we can write $\mathbf{\Theta}=(\mathbf{P}_0 \boldsymbol{\phi})^T \mathbf{V} \boldsymbol{\Delta} \mathbf{V}^{-1} (\mathbf{P}_0 \boldsymbol{\phi})$ where $\mathbf{V}=\operatorname{blkdiag}\left\{V_i\right\}_{i=2}^N $ and 
$\boldsymbol{\Delta}=\operatorname{blkdiag}\left\{\mathbf{\Delta}_i\right\}_{i=2}^N$.
Moreover, $\|\mathbf{\Delta}\| = 1 - {r^2\lambda_l\rho \tau \beta}/{2}$
when $r^2\lambda_u\rho \tau \beta <2$.
%
$\widehat{\mathbf{V}}^{-1}=\mathbf{V}^{-1}(\mathbf{P}_0 \boldsymbol{\phi})$, where $\mathbf{P}_0$ is a permutation matrix and $\boldsymbol{\phi}$ is an orthogonal matrix, yields
$
\widehat{\mathbf{d}}_{k+1}=\mathbf{\Delta} \widehat{\mathbf{d}}_k- \widehat{\mathbf{h}}_{k},
$
where $
\widehat{\mathbf{d}}_k = \widehat{\mathbf{V}}^{-1}
\begin{bmatrix}
\widehat{\mathbf{Q}}^T  \mathbf{X}_{k};
   \widehat{\mathbf{Q}}^T \mathbf{Y}_{k};
   \widehat{\mathbf{Q}}^T \Tilde{\mathbf{Y}}_{k}
\end{bmatrix}
$,
$
\widehat{\mathbf{h}}_{k} =  \widehat{\mathbf{V}}^{-1} \widehat{\mathbf{Q}}^T  \mathbf{h}_k
$.
\end{proof}

Now we derive an upper bound for $\mathrm{E}[\|g_i(\phi_{i, k}^t)-\nabla f_i({\phi}_i^k)\|^2]$, which is the variance of the gradient estimator. Define $t_i^k$ as the averaged optimality gap of the auxiliary variables of $\left\{\mathbf{r}_{i, j}^k\right\}_{j=1}^{m_i}$ at node $i$ as follows:
\begin{align*}
&t_{i,k}^{t}=\frac{1}{m_i} \sum_{h=1}^{m_i}\|r_{i, h,k}^{t}-\bar{x}_k\|^2, 
\\&
t^{t}_k=\sum_{i=1}^N t_{i,k}^{t} = \frac{1}{m_i} \sum_{h=1}^{m_i}\|\mathbf{r}_{h,k}^{t}-\bar{\mathbf{X}}_k\|^2, 
\\&{t}_k=\sum_{t=0}^{\tau-1}t^{t}_k =\sum_{t=0}^{\tau -1 } \sum_{i=1}^N t_{i,k}^{t}.
\end{align*}
Then using  $ \mathbb{E} [ \|a- \mathbb{E}[a] \|^2 ] \leq \mathbb{E} [ \| a \|^2 ]$ with $ a = \nabla f_{i, h}( \phi_{i, k}^t )-\nabla f_{i, h}(r_{i, h, k}^{t}) $,
\begin{equation}
\begin{aligned}
& \mathrm{E}\left[\|g_i(\phi_{i, k}^t) - \nabla f_i\left( \phi_{i, k}^t \right)\|^2 \right]
\\& \leq \mathrm{E}\left[\left\|\nabla f_{i, s_i^k}( \phi_{i, k}^t )-\nabla f_{i, s_i^k}\left(\mathbf{r}_{i, s_i^k}^k\right)\right\|^2 \right] 
\\& =\frac{1}{m_i} \sum_{j=1}^{m_i} \|(\nabla f_{i, j} (\phi_{i, k}^t)-\nabla f_{i, j}(\overline{x}_k))  
\\
&+   (\nabla f_{i, j}\left(\overline{x}_k\right)-\nabla f_{i, j}(\mathbf{r}_{i, j}^{k}) ) \|^2 
\\
& \leq 2 L^2\left\|\phi_{i, k}^t -\bar{x}_k\right\|^2+2 L^2 t_i^{k}, 
\end{aligned}
\end{equation}
it follows that
\begin{equation}
\sum_i \sum_t\left\|   g_i\left(\phi_{i, k}^t\right) -   \nabla f_i(\phi_{i, k}^t)   \right\|^2 \leq 2L^2 \|\widehat{\mathbf{\Phi}}_k \|^2 + 2 L^2 t_k.
\end{equation}

\begin{lemma}
Let Assumptions hold; when $\beta <  \frac{2}{r^2\tau\lambda_u\rho}$ and 
\begin{equation} \label{eq:gamma_phi}
4 \gamma^2 \tau(2L^2+ L^2) \leq \frac{1/4}{\tau -1},
\end{equation} we have
\begin{equation}  \label{phi_vr}
\begin{aligned}
&\mathbb{E} [\|\widehat{\mathbf{\Phi}}_k\|^2 ]
\leq \left( \frac{72 \beta \tau^2}{r^2 \lambda_l\rho}  + 144\tau^3\beta^2 \right) \mathbb{E} [\| \widehat{\mathbf{d}}_k \|^2 ] \\& + 16 \tau^3\gamma^2 N  \mathbb{E} [\left\| \nabla F(\bar{x}_k)\right\|^2]  + 32 \tau^2 \gamma^2  L^2  \mathbb{E} [t_k].
\end{aligned}
\end{equation}
\end{lemma}
\begin{proof}
Using similar analysis with \cite[Lemma 2]{alghunaim_local_2023} and  \eqref{phi_vr}, we obtain
\begin{equation} \label{eq:phi_vr_0}
\begin{aligned}
& \mathbb{E} [ \|\Phi_k^{t+1}-\bar{\mathbf{X}}_k \|^2]  \leq  \mathbb{E} [ \|\mathbf{X}_{k}-\bar{\mathbf{X}}_k \|^2 ]+ 16\tau^2\beta^2  \mathbb{E} [ \|  \mathbf{\bar{Y}_{k}} \|^2 ]
\\&+ 16\tau^2\beta^2  \mathbb{E} [ \| \mathbf{Y}_{k} - \mathbf{\bar{Y}_{k}} \|^2 ] + 32 \tau^2 \gamma^2  L^2   \mathbb{E} [t^{t}_k ]
\end{aligned}
\end{equation}
Summing  \eqref{eq:phi_vr_0} over $t$ and using \eqref{X_Y_d}   we obtain \eqref{phi_vr}.
\end{proof}

\smallskip

The following lemma provides the bound on $t_k$.
\begin{lemma} \label{lemma:t_k}
Let  Assumptions~\ref{as:local-costs} and  \ref{as:graph} hold, $\left\{t_k\right\}$ be the iterates generated by \alg. If  $\gamma$ satisfies \eqref{eq:gamma_phi}, \eqref{eq.gamma_saga_2_0} and \eqref{eq.gamma_saga_2_1},  $\beta <  \frac{2}{r^2\tau\lambda_u\rho}$, we have for all $k \in \N$:
\begin{equation}\label{t_vr}
\begin{aligned}
& \mathbb{E}[t_{k}]\leq 2(s_0+s_1 )\mathbb{E} [\| \widehat{\mathbf{d}}_k\|^2] + 2s_2 \mathbb{E} [\| \nabla F(\bar{x}_k)\|^2 ],
\end{aligned}
\end{equation}
where
\begin{align}
&  s_0 = \frac{36\beta \tau^2 m_u }{\lambda_l\rho} +\frac{144 \tau^2 m_u}{m_l}\beta^2 \nonumber
\\& s_1 = \left( \frac{72\beta \tau^2}{r^2\lambda_l\rho}  +144\tau^3\beta^2 \right) \frac{8m_u \tau}{m_l} \label{eq.s}
\\& s_2 = \frac{16N \gamma^2 m_u \tau^2}{m_l} + \frac{8m_u \tau}{m_l}16 \tau^3\gamma^2 N. \nonumber
\end{align}
\end{lemma}

\smallskip

\begin{proof}
From Algorithm~\ref{alg:lt-saga-admm}, 
$\forall k,  r_{i, h, k}^{ t+1} = r_{i, h,k}^{k}$ with probability $1-\frac{1}{m_i}$ and $r_{i, h,k}^{t+1} =  \phi_{i,k}^{t+1}$ with probability $\frac{1}{m_i}$, therefore, 
\begin{align*}
& \mathbb{E}[t^{t+1}_k] 
 = \frac{1}{m_i} \sum_{h=1}^{m_i} \mathbb{E}[ \|\mathbf{r}_{h,k}^{t+1}-\bar{\mathbf{X}}_{k}\|^2]
\\& =\frac{1}{m_i} \sum_{h=1}^{m_i} \mathbb{E}[ (1-\frac{1}{m_i})\|\mathbf{r}_{h,k}^{t}-\bar{\mathbf{X}}_{k} \|^2  +\frac{1}{m_i}\|\Phi_{k}^{t+1}-\bar{\mathbf{X}}_{k} \|^2]
\\
& = \left(1-\frac{1}{m_i}\right) \frac{1}{m_i} \sum_{h=1}^{m_i}\mathbb{E}\left[ \|\mathbf{r}_{h,k}^{t}-\bar{\mathbf{X}}_{k} \|^2 \right] \\ &+\frac{1}{m_i} \mathbb{E} [\|\Phi_{k}^{t+1}-\bar{\mathbf{X}}_{k} \|^2].
\end{align*}
Denote
$q^t_k = \beta \mathbf{Y}_{k} - \gamma( G(\mathrm{\Phi}_k^t) -  \nabla \mathrm{F}(\bar{\mathbf{X}}_k) ),$
we have
$\left\|\mathrm{\Phi}_{k}^{t+1}-\bar{X}_{k}\right\|^2 = \left\| \mathrm{\Phi}_{k}^{t+1}-  \mathrm{\Phi}_{k}^{t} + \mathrm{\Phi}_{k}^{t} -\bar{X}_{k}\right\|^2 
\leq 2 \| \Phi_{k}^{t} -\bar{X}_{k} \|^2  + 2\| q^t_k\|^2,
$ and
\begin{align*}
&\mathbb{E} [\| q^t_k \|^2] 
\\& \leq 2 \gamma^2  \mathbb{E} [ \|    G(\mathrm{\Phi}_k^t) -  \nabla \mathrm{F}(\bar{\mathbf{X}}_k) \|^2 ] + 2\beta^2   \mathbb{E} [ \left\| \mathbf{Y}_{k}  \right\|^2 ] \nonumber  
\\& \leq4 \gamma^2( 2L^2+ L^2  )  \mathbb{E} [ \left\|\mathrm{\Phi}_{k}^t-\bar{\mathbf{X}}_k\right\|^2 ] +2 \beta^2   \mathbb{E} [ \left\|  \mathbf{Y}_{k}  \right\|^2 ]
\\& + 4\gamma^2 
( 2  L^2  \mathbb{E} [ t^{t}_k ] )
\\& \leq 12 \gamma^2 L^2 \left\|\mathrm{\Phi}_{k}^t-\bar{\mathbf{X}}_k\right\|^2 + 4\gamma^2 N \| \nabla {F}(\bar{{x}}_k)\|^2 
\\&+ 4\beta^2 \|  \mathbf{Y}_{k}  -  \mathbf{\bar{Y}}_{k}  \|^2 + 8 \gamma^2 L^2  \mathbb{E} [ t^{t}_k ],
\end{align*}
it follows that
\begin{align}
& \mathbb{E}[t^{t+1}_k] 
= (1-\frac{1}{m_i}) \frac{1}{m_i} \sum_{h=1}^{m_i}\|\mathbf{r}_{h,k}^{t}-\bar{X}_{k}\|^2 
\\&+\frac{1}{m_i}\|\mathrm{\Phi}_{k}^{t+1}-\bar{X}_{k}\|^2 \nonumber
\\& \leq (1-\frac{1}{m_i})t^{t}_k +\frac{1}{m_i} 
\left(2 \| \mathrm{\Phi}_{k}^{t} -\bar{X}_{k} \|^2  + 2\| q^t_k\|^2\right) \nonumber
\\& \leq \left(1-\frac{1}{m_u} + \frac{16 \gamma^2  L^2}{m_l}  \right)\mathbb{E}[t^{t}_k] 
\\&+ \left(  \frac{2}{m_l} +  \frac{24\gamma^2L^2}{m_l}  \right)\mathbb{E} [  \|\mathrm{\Phi}_{k}^t-\bar{\mathbf{X}}_k\|^2  \nonumber
\\& + \frac{72}{m_l}\beta^2\| \widehat{\mathbf{d}}_k \|^2 + \frac{8N}{m_l}\gamma^2\| \nabla F(\bar{x}_k) \|^2 \nonumber
\\& \leq \left(1-\frac{1}{2m_u}  \right)\mathbb{E}[t^{t}_k] + \frac{4}{m_l} \mathbb{E} [  \|\mathrm{\Phi}_{k}^t-\bar{\mathbf{X}}_k\|^2 
\\&+ \frac{72}{m_l}\beta^2\| \widehat{\mathbf{d}}_k \|^2 + \frac{8N}{m_l}\gamma^2\| \nabla F(\bar{x}_k) \|^2 \label{t_k_update}
\end{align}
where the last inequality holds when 
\begin{equation} \label{eq.gamma_saga_2_0}
\frac{16 \gamma^2  L^2}{m_l} < \frac{1}{2m_u},
    \quad
   {24\gamma^2L^2}< 2. 
\end{equation}
Iterating \eqref{t_k_update} for $t=0,...,\tau-1$ then yields:
\begin{align*}
& \mathbb{E} [  t^{t}_k ]
\leq(1-\frac{1}{2m_u})^t \mathbb{E} [ \|\mathbf{X}_{k}-\bar{\mathbf{X}}_k \|^2]
\\& +\frac{72}{m_l}\beta^2 \sum_{l=0}^{t-1}  (1-\frac{1}{2m_u})^{t-1-l} \mathbb{E}[ \| \widehat{\mathbf{d}}_k \|^2] 
\\&+ \frac{8N \gamma^2}{m_l} \sum_{l=0}^{t-1}  (1-\frac{1}{2m_u})^l  \| \nabla F(\bar{x}_k) \|^2
 \\&+  \frac{4}{m_l}\sum_{l=0}^{t-1}  (1-\frac{1}{2m_u})^{t-1-l} \mathbb{E} [  \|\Phi_{k}^{l}-\bar{\mathbf{X}}_k\|^2
 \\&\leq \frac{36\beta \tau m_u }{\lambda_l\rho} \mathbb{E} [\| \widehat{\mathbf{d}}_k\|^2]
+ \frac{16N \gamma^2 m_u \tau}{m_l}  \| \nabla F(\bar{x}_k) \|^2
 \\&+  \frac{8m_u \tau}{m_l}\mathbb{E} [  \| \mathrm{\Phi}_{k}^{l}-\bar{\mathbf{X}}_k\|^2
 +\frac{144 m_u \tau}{m_l}\beta^2 \mathbb{E}[ \| \widehat{\mathbf{d}}_k \|^2].
\end{align*}
Summing the above relation over $t = 0, 1,..., \tau -1$ we get: 
\begin{align*}
&\mathbb{E} [  t_{k} ] \\&  \leq 
 \left( \frac{36\beta \tau^2 m_u }{\lambda_l\rho} +\frac{144 \tau^2 m_u}{m_l}\beta^2\right)
\mathbb{E} [\| \widehat{\mathbf{d}}_k\|^2] + \frac{8m_u \tau}{m_l} \|\widehat{\mathbf{\Phi}}_k \|^2 
\\&+ \frac{16N \gamma^2 m_u \tau^2}{m_l} \| \nabla F(\bar{x}_k) \|^2,
\end{align*}
and using \eqref{phi_vr} then yields
\begin{align*}
\mathbb{E} [  t_{k} ] &\leq (s_0+s_1 )\mathbb{E} [\| \widehat{\mathbf{d}}_k\|^2] + s_2 \mathbb{E} [\| \nabla F(\bar{x}_k)\|^2 ] \\&+  \frac{8m_u \tau}{m_l} 32 \tau^2 \gamma^2  L^2 \mathbb{E} [t_k],
\end{align*}
where $s_0$, $s_1$ and $s_2$ are defined in \eqref{eq.s}.
Letting 
\begin{equation}  \label{eq.gamma_saga_2_1}
 \frac{8m_u \tau}{m_l} 32 \tau^2 \gamma^2  L^2  < \frac{1}{2}, 
\end{equation}
and thus \eqref{t_vr} holds.
\end{proof}
 
\begin{lemma}
Let Assumptions~\ref{as:local-costs},  \ref{as:graph}, \ref{as:compressor} and \ref{as:independent compressor} hold, set $r^2\lambda_u\rho \tau \beta =1$,  $\gamma$ satisfies the conditions in Lemma~\ref{lemma:t_k},
it holds that
\begin{equation} \label{eq:d_k}
\begin{aligned}
&\mathbb{E} [ \|\widehat{\mathbf{d}}_{k+1}\|^2 ] 
\leq 
  ( \delta + \frac{ q_0\| \widehat{\mathbf{V}}^{-1} \|^2}{1-\delta}) \mathbb{E} [ \|\widehat{\mathbf{d}}_{k}\|^2 ]
 \\&+ \frac{q_1\| \widehat{\mathbf{V}}^{-1} \|^2}{1-\delta} \mathbb{E} [  \| \bar{x}_k - x^* \|^2 ] 
 \\&+ \frac{a_5\| \widehat{\mathbf{V}}^{-1} \|^2}{1-\delta}  \mathbb{E}[\|  \mathbf{X}_{k} - \mathbf{U}_{k}  \|^2   ]\\&+\frac{a_6\| \widehat{\mathbf{V}}^{-1} \|^2}{1-\delta}   \mathbb{E}[\| (\mathbf{I} + \mathbf{P})( \mathbf{Z}_{k} - \mathbf{S}_{k}  )\|^2   ]
\end{aligned} 
\end{equation}
where 
\begin{equation}
\begin{aligned}
    &a_0 = \gamma^2 \bigg( ( 1+ 3\rho^2 \|\Tilde{\mathbf{L}} \|^2)  6\tau L^2  + 54 L^4 \frac{\gamma^2 \tau }{\beta^2}
    \\& \quad + 72 r^4 (p-1) \lambda_u^2\frac{\tau  L^2}{\eta}   \bigg)
    \\& a_1 = \gamma^2 \bigg( ( 1+ 3\rho^2 \|\Tilde{\mathbf{L}} \|^2)  4\tau L^2 + 36 L^4 \frac{\gamma^2}{\beta^2} \tau 
    \\& \quad + 48 r^4 (p-1) \lambda_u^2\frac{\tau  L^2}{\eta}   \bigg)
    \\& a_3 = \gamma^2\bigg( 18 L^2 \frac{\gamma^2}{\beta^2}N \tau^2+ 24 r^4 (p-1) \lambda_u^2 \tau^2  N \bigg)
        \\& a_4 = \frac{8 (p-1)}{3\rho^2\eta}  + \frac{15d_u (p- 1) }{ \lambda_l \rho} 
\\& a_5 = 6 r^4 (p-1) \lambda_u^2 (1 - \eta + \eta^2(p - 1)) 
    \\&+  \frac{3r^2d_u}{4} \frac{5}{2}(p- 1)   8 \rho^2 r^2 2\lambda_u^2  (p -1) 
    \\& a_6 =  \frac{3 r^2 d_u}{8}(p- 1) 
\\& q_0 = ( \frac{8 \tau \lambda_u}{ \lambda_l K^2}  +  \frac{16\tau}{9 K^2})a_0 +  64 \tau^2 \gamma^2  L^2 a_0 (s_0+s_1 )
\\& \quad + 2a_1(s_0+s_1 )+a_4
\\& q_1 = 16 \tau^3\gamma^2 N  L^2 a_0 +  64 \tau^2 \gamma^2  L^4 a_0  s_2 + 2a_1 s_2 L^2 +a_3,
\\& \delta =1- \frac{\lambda_l}{2\lambda_u}, \quad K=r^2 \rho \lambda_u
\end{aligned}    
\end{equation}
\end{lemma}
\begin{proof}
According to Lemma~\ref{lem:devitaion_aver},  $\delta = \|\mathbf{\Delta}\|= 1 - \frac{r^2\lambda_l\rho \tau \beta}{2} =1- \frac{\lambda_l}{2\lambda_u}$,  $(1-\delta)^2 = \frac{\lambda_l^2}{4\lambda_u^2}$.
We consider the compression error, using Assumptions~\ref{as:compressor} and \ref{as:independent compressor}, 
\begin{equation}
\begin{aligned}
&\mathbb{E}[ \| (\mathbf{I} - \mathbf{P}) ( \mathbf{Z}_{k} - \hat{\mathbf{Z}}_{k}) \|^2 ]
\\& = \mathbb{E}[ \| (\mathbf{I} - \mathbf{P})  ( \mathbf{Z}_{k} - \mathbf{S}_{k} -  \mathcal{C} ( \mathbf{Z}_{k} - \mathbf{S}_{k} ) )\|^2 ] 
\\& = \mathbb{E}[ \| (\mathbf{I} - \mathbf{P})  (\mathbf{Z}_{k} - \mathbf{S}_{k})   -    \mathcal{C} ( (\mathbf{I} - \mathbf{P})(Z_k - S_k ) )\|^2   \\&+  \|   \mathcal{C} (  (\mathbf{I} - \mathbf{P}) (\mathbf{Z}_{k} - \mathbf{S}_{k} ) )   -  (\mathbf{I} - \mathbf{P})  \mathcal{C} ( \mathbf{Z}_{k} - \mathbf{S}_{k} )  \|^2 ]
\\& \leq (p -1) \|  (\mathbf{I} - \mathbf{P})  (\mathbf{Z}_{k} - \mathbf{S}_{k})\|^2  \\&+ \mathbb{E}[  \|   \mathcal{C} (  (\mathbf{I} - \mathbf{P}) (\mathbf{Z}_{k} - \mathbf{S}_{k} ) )   -  (\mathbf{I} - \mathbf{P})  \mathcal{C} ( \mathbf{Z}_{k} - \mathbf{S}_{k} )  \|^2 ],
\end{aligned}
\end{equation}
by the definition of the permutation matrix $\mathbf{P}$, $\mathbf{P}\mathcal{C} ( \mathbf{Z}_{k} - \mathbf{S}_{k} )= \mathcal{C} ( \mathbf{P} (\mathbf{Z}_{k} - \mathbf{S}_{k}) ) $. Since $\forall z_1, z_2$, under Assumptions~\ref{as:compressor} and \ref{as:independent compressor} we have
\begin{equation}
\begin{aligned}
&\mathbb{E}[ \| \mathcal{C}(z_1-z_2) - ( \mathcal{C}(z_1) -   \mathcal{C}(z_2) ) \|^2 ]
\\& = \mathbb{E}[ \| \mathcal{C}(z_1-z_2) \|^2 + \|  \mathcal{C}(z_1) - \mathcal{C}(z_2) \|^2  
\\&- 2 \mathcal{C}(z_1-z_2) (\mathcal{C}(z_1) - \mathcal{C}(z_2) ) ]
\\& \leq p (z_1 - z_2)^2 + p(z_1^2 + z_2^2) - 2 z_1z_2 - 2(z_1 -z_2)^2
\\& =  \frac{3}{2} (p- 1) (z_1-z_2)^2 + \frac{p - 1}{2}(z_1 + z_2)^2,
\end{aligned}
\end{equation}
we obtain that
\begin{equation*}
\begin{aligned}
&\mathbb{E}[ \| \mathcal{C} (  (\mathbf{I} - \mathbf{P})(\mathbf{Z}_{k} - \mathbf{S}_{k} ) )   -   (\mathbf{I} - \mathbf{P})  \mathcal{C} ( \mathbf{Z}_{k} - \mathbf{S}_{k})  \|^2 ]
\\& \leq \frac{3}{2} (p- 1) \|  (\mathbf{I} - \mathbf{P})(\mathbf{Z}_{k} - \mathbf{S}_{k} ) \|^2 
\\&+ \frac{p-1}{2} \|  (\mathbf{I} + \mathbf{P})(\mathbf{Z}_{k} - \mathbf{S}_{k}) \|^2. 
\end{aligned}
\end{equation*}
According to \eqref{eq:compact-compress} and \eqref{eq:compact-admm-z},
\begin{equation}
\begin{aligned}
&(\mathbf{I} - \mathbf{P} )( \mathbf{Z}_{k} - \mathbf{S}_{k}) = (\mathbf{I} - \mathbf{P} )( \rho r \mathbf{A}\mathbf{X}_{k} - \rho r(\mathbf{I} - \mathbf{P}) \mathbf{A} \hat{\mathbf{X}}_{k})
\\& = \rho r ( 2(\mathbf{I} - \mathbf{P} )\mathbf{A} ( \mathbf{X}_{k} - \hat{\mathbf{X}}_{k} )  - (\mathbf{I} - \mathbf{P} )\mathbf{A} ( \mathbf{X}_{k} - \bar{\mathbf{X}}_{k} ) ),
\end{aligned}
\end{equation}
therefore,
\begin{equation}
\begin{aligned}
&\mathbb{E}[ \| (\mathbf{I} - \mathbf{P}) ( \mathbf{Z}_{k} - \hat{\mathbf{Z}}_{k}) \|^2 ] 
\\& \leq  \frac{5}{2}(p- 1) (  8 \rho^2 r^2 \| (\mathbf{I} - \mathbf{P} ) \mathbf{A}  \|^2 (p-1) \|  \mathbf{X}_{k} -  {\mathbf{U}}_{k}  \|^2 \\&+ 2 \rho^2 r^2 \| (\mathbf{I} - \mathbf{P} ) \mathbf{A}  \|^2 \|  \mathbf{X}_{k} - \bar{\mathbf{X}}_k \|^2  )  
\\&+ \frac{1}{2}(p- 1)  \|  (\mathbf{I} + \mathbf{P})(\mathbf{Z}_{k} - \mathbf{S}_{k} ) \|^2.
\end{aligned} 
\end{equation}
Similarly,  we can derive that 
\begin{equation}
\begin{aligned}
&\mathbb{E}[   \| (\mathbf{I} +\mathbf{P}) ( \mathbf{Z}_k - \mathbf{\hat{Z}}_k )\|^2] 
\leq\frac{5}{2}(p- 1)  \| (\mathbf{I} + \mathbf{P}) (\mathbf{Z}_{k} - \mathbf{S}_{k} )   \|^2 \\& + \frac{p-1}{2} \|  (\mathbf{I}- \mathbf{P})(\mathbf{Z}_{k} - \mathbf{S}_{k} ) \|^2.
\end{aligned}
\end{equation}
From \eqref{eq:compact-admm}, we can also derive that
\begin{equation}
\begin{aligned}
&(\mathbf{I} + \mathbf{P})  (\mathbf{Z}_{k+1} - \mathbf{S}_{k+1})
\\& = \rho r (\mathbf{I} +\mathbf{P}) \mathbf{A} ( \mathbf{X}_{k+1} - \mathbf{X}_k ) + (\mathbf{I} +\mathbf{P}) ( \mathbf{Z}_k - \mathbf{\hat{Z}}_k )
\end{aligned}
\end{equation}
From \eqref{eq:compact-admm-x}, it can be obtained that
\begin{equation}
    \begin{aligned}
&\|  \mathbf{X}_{k+1} - \mathbf{X}_k  \|^2   = \|  \sum_{t=0}^{\tau-1}q^t_k  \|^2 \leq \tau \sum_{t=0}^{\tau-1}\| q^t_k  \|^2
\\& \leq  \tau \sum_{t=0}^{\tau-1} ( 12 \gamma^2 L^2 \left\|\mathbf{\Phi}_{k}^t-\bar{\mathbf{X}}_k\right\|^2 + 4\gamma^2 N \| \nabla {F}(\bar{{x}}_k)\|^2 
\\&+ 4\beta^2 \|  \mathbf{Y}_{k}  -  \mathbf{\bar{Y}}_{k}  \|^2 + 8 \gamma^2 L^2  \mathbb{E} [ t^{t}_k ])
\\& \leq 12 \tau \gamma^2 L^2 \|\widehat{\mathbf{\Phi}}_k\|^2 + 4\tau^2 \gamma^2 N \| \nabla {F}(\bar{{x}}_k)\|^2 
\\&+ 36 \tau^2 \beta^2 \| \widehat{\mathbf{d}}_k\|^2 + 8 \tau \gamma^2 L^2  \mathbb{E} [ t_k ],
    \end{aligned}
\end{equation}
so
\begin{equation} \label{eq:compress_z}
\begin{aligned}
  & \mathbb{E} [ \| (\mathbf{I} + \mathbf{P})  (\mathbf{Z}_{k+1} - \mathbf{S}_{k+1})\|^2 ]
\\& = \mathbb{E} [ \| \rho r (\mathbf{I} +\mathbf{P}) \mathbf{A} ( \mathbf{X}_{k+1} - \mathbf{X}_k ) \|^2] \\ &+ \mathbb{E} [ \|(\mathbf{I} +\mathbf{P}) ( \mathbf{Z}_k - \mathbf{\hat{Z}}_k ) \|^2]
\\& \leq \frac{5(p-1)}{2} \mathbb{E} [\| (\mathbf{I} + \mathbf{P}) (\mathbf{Z}_{k} - \mathbf{S}_{k} )   \|^2 ]
\\&+   \rho^2 r^2 \| (\mathbf{I} +\mathbf{P}) \mathbf{A}  \|^2 ( 
12 \tau \gamma^2 L^2 \|\widehat{\mathbf{\Phi}}_k\|^2 \\&+ 4\tau^2 \gamma^2 N \| \nabla {F}(\bar{{x}}_k)\|^2 
+ 36 \tau^2 \beta^2 \| \widehat{\mathbf{d}}_k\|^2 + 8 \tau \gamma^2 L^2  \mathbb{E} [ t_k ]
  ) \\&+ \frac{p-1}{2} (   8 \rho^2 r^2\| (\mathbf{I} - \mathbf{P} ) \mathbf{A}  \|^2 (p - 1) \|  \mathbf{X}_{k} -  {\mathbf{U}}_{k}  \|^2 
  \\&+  2 \rho^2 r^2\| (\mathbf{I} - \mathbf{P} ) \mathbf{A}  \|^2 \|  \mathbf{X}_{k} - \bar{\mathbf{X}}_k \|^2  ).
\end{aligned}
\end{equation}
According to \eqref{eq:compact-compress} and \eqref{eq:compact-admm-x},
\begin{equation} \label{eq:compress_x}
\begin{aligned}
&\mathbb{E}[\|  \mathbf{X}_{k+1} - \mathbf{U}_{k+1}  \|^2   ]
\\& = \mathbb{E} [ \|  \mathbf{X}_{k+1} - \mathbf{X}_{k} +  (1-\eta)( \mathbf{X}_{k} - \mathbf{U}_{k} ) + \eta (\mathbf{X}_{k} - \hat{\mathbf{X}}_{k} )   \|^2 ]
\\& =  \mathbb{E} [ \|  \mathbf{X}_{k+1} - \mathbf{X}_{k}  +  (1-\eta)( \mathbf{X}_{k} - \mathbf{U}_{k} ) \|^2 + \\& \eta^2\|\mathbf{X}_{k} - \hat{\mathbf{X}}_{k}    \|^2 ]
\\& \leq \frac{1}{\eta} \| \mathbf{X}_{k+1} - \mathbf{X}_{k}  \|^2 
 +  (1- \eta + \eta^2 ( p-1 ) ) \mathbb{E} [\|  \mathbf{X}_{k} - \mathbf{U}_{k}\|^2]
 \\& \leq (1 - \eta + \eta^2(p - 1))
\mathbb{E} [\| \mathbf{X}_{k} - \mathbf{U}_{k} \|^2] \\ &+ \frac{1}{\eta}( 12 \tau \gamma^2 L^2 \|\widehat{\mathbf{\Phi}}_k\|^2 \\&+ 4\tau^2 \gamma^2 N \| \nabla {F}(\bar{{x}}_k)\|^2 
+ 36 \tau^2 \beta^2 \| \widehat{\mathbf{d}}_k\|^2 + 8 \tau \gamma^2 L^2  \mathbb{E} [ t_k ]).
\end{aligned}
\end{equation}
Using  $r^2\lambda_u\rho \tau \beta =1$,  $r^2 \rho \lambda_u = K$ and \eqref{eq:h_k_0}, we can obtain that
\begin{equation}
\begin{aligned}
& \|{\mathbf{h}}_{k}\|^2  
 \leq  q_0 \| \widehat{\mathbf{d}}_k \|^2  +q_1 \| \bar{x}_k - x^* \|^2 \\& + a_6 \|  (\mathbf{I} + \mathbf{P})(\mathbf{Z}_{k} - \mathbf{S}_{k}  ) \|^2 
+ a_5 \| \mathbf{X}_{k} - \mathbf{U}_{k} \|^2.
\end{aligned}
\end{equation}
Recalling \eqref{eq:delta-hat},  $ 
\|\widehat{\mathbf{d}}_{k+1}\|^2  \leq  \frac{1}{\|\mathbf{\Delta}\|}\|  \mathbf{\Delta}\|^2  \|\widehat{\mathbf{d}}_{k}\|^2 +
\frac{1}{1-\|\mathbf{\Delta}\|} \|\mathbf{\widehat{h}}_{k}\|^2 
 $, and $\|\mathbf{\widehat{h}}_{k}\|^2 \leq  \| \widehat{\mathbf{V}}^{-1} \|^2 \|\mathbf{h}_{k}\|^2 $ yields \eqref{eq:d_k}.
\end{proof}

\subsection{Deviation of $\bar{x}_{k}$ from the optimal solution $x^*$}
From \eqref{eq:compact-admm-x}, we can derive that 
\begin{equation}
\begin{aligned}
&\mathrm{E} \|\bar{x}_{k+1} - x^*\|^2  \leq \mathrm{E} \| \bar{x}_{k}  - x^* \|^2 + \gamma^2\mathrm{E}\| \sum_{t=0}^{\tau -1} \overline{G}(\mathrm{\Phi}_k^t) \|^2
\\&- 2\frac{\gamma}{N} \langle  \bar{x}_{k}  - x^*,   \sum_{t=0}^{\tau -1} \sum_{i=1}^{N}\nabla f_i(\phi_{i, k}^t) \rangle
\end{aligned}
\end{equation}
Since $\langle(z-y), \nabla g(x)\rangle \geq g(z)-g(y)+\frac{\mu}{4}\|y-z\|^2-L\|z-x\|^2, \forall x, y, z \in R^n$
for any L-smooth and $\mu$-strongly convex function $g$ \cite{nesterov2013introductory}, we have
\begin{equation}
\begin{aligned}
&-\frac{2 \gamma}{N} \sum_i \sum_t\left\langle\left(\bar{x}_k-x^{\star}\right), \nabla f_i\left(\phi_{i, k}^t\right)\right\rangle \\& 
\leq \frac{2 \gamma}{N} \sum_i \sum_t(f_i\left(x^{\star}\right)-f_i\left(\bar{x}_k\right)-\frac{\mu}{4}\left\|\bar{x}_k-x^{\star}\right\|^2
\\&+L\left\|\bar{x}_{k}-\phi_{i, k}^t\right\|^2) 
\\& = -{2\gamma\tau} (  F\left(\bar{x}_{k}\right)-F\left(x^{\star}\right)+\frac{\mu}{4}\left\|\bar{x}_k-x^{\star}\right\|^2  ) +\frac{2 \gamma L}{N}\|\widehat{\mathbf{\Phi}}^k \|^2,
\end{aligned}
\end{equation}
it follows that
\begin{equation}\label{eq:x_star}
\begin{aligned}
& \mathrm{E} [\|\bar{x}^{k+1}-x^{\star}\|^2] \leq 
\left(1-\frac{\mu \tau \gamma}{2}  \right)  \mathbb{E} \left\|\bar{x}^k-x^{\star}\right\|^2 \\&+\frac{2 \gamma L}{N}     \mathbb{E} \|\widehat{\mathrm{\Phi}}^k\|^2
 - 2 \gamma \tau   \left(F\left(\bar{x}^k\right)-F\left(x^{\star}\right)\right) \\&
 +\gamma^2    \|\sum_{t=0}^{\tau -1} \overline{G}(\mathrm{\Phi}_k^t)\|^2
 \\& \leq \bigg(1-\frac{\mu \tau \gamma}{2}  + (\frac{2 \gamma L}{N} + \frac{9 \gamma^2 \tau  L^2}{N} ) 16 \tau^3\gamma^2 N  L^2 
 \\&+  (\frac{2 \gamma L}{N} + \frac{9 \gamma^2 \tau  L^2}{N} )64 \tau^2 \gamma^2  L^2 s_2 L^2 
 \\&+  \gamma^2  \frac{12 \tau  L^2}{N} s_2 L^2 + 3\gamma^2 \tau^2L^2   \bigg) \mathrm{E} [\|\bar{x}^k-x^{\star}\|^2]
 \\& + \bigg( (\frac{2 \gamma L}{N} + \frac{9 \gamma^2 \tau  L^2}{N} )( \frac{72\beta \tau^2}{\lambda_l\rho}  + 144\tau^3\beta^2 ) 
 \\& + (\frac{2 \gamma L}{N} + \frac{9 \gamma^2 \tau  L^2}{N} )64\tau^2 \gamma^2  L^2 (s_0+s_1)
 \\& + \gamma^2  \frac{12 \tau  L^2}{N} (s_0+ s_1)\bigg) \mathbb{E} [\| \widehat{\mathbf{d}}_k \|^2]
 \\& = w_0\mathrm{E} [\|\bar{x}^k-x^{\star}\|^2] + w_1  \mathbb{E} [\| \widehat{\mathbf{d}}_k \|^2].
\end{aligned}
\end{equation}

\subsection{Proof of Theorem~\ref{th:convergence}}
Let Assumptions~\ref{as:local-costs},  \ref{as:graph}, \ref{as:compressor} and \ref{as:independent compressor} hold, $\gamma$ satisfies \eqref{eq:gamma_phi}, \eqref{eq.gamma_saga_2_0} and \eqref{eq.gamma_saga_2_1}, $r^2\lambda_u\rho \tau \beta =1$. Based on the above relations, we can derive that 
\begin{equation}
    \mathbf{T}_{k+1} \leq \mathbf{\Xi} \mathbf{T}_k,
\end{equation}
where 
$\mathbf{T}_k = [
    \mathbb{E}[\|\bar{x}^{k+1}-x^{\star}\|^2]; \mathbb{E}[\|\widehat{\mathbf{d}}_{k+1}\|^2];\mathbb{E}[\|  \mathbf{X}_{k+1} - \mathbf{U}_{k+1}  \|^2   ]; \mathbb{E}[\| (\mathbf{I}+ \mathbf{P})  (\mathbf{Z}_{k+1} - \mathbf{S}_{k+1}  )\|^2]
]$.
The diagonal elements of $\mathbf{\Xi}$ are $\mathbf{\Xi}_{11}=w_0$, $\mathbf{\Xi}_{22}=\delta + \frac{ q_0\| \widehat{\mathbf{V}}^{-1} \|^2}{1-\delta}$, $\mathbf{\Xi}_{33} = 1 - \eta + \eta^2(p - 1)$, $\mathbf{\Xi}_{44} = \frac{5(p-1)}{2}$, 
the other elements of the nonnegative matrix  $\mathbf{\Xi}$ can be obtained by the inequalities \eqref{phi_vr}, \eqref{t_vr}, \eqref{eq:d_k}, \eqref{eq:compress_z},  \eqref{eq:compress_x} and  \eqref{eq:x_star}.
When the spectral radius of $\mathbf{\Xi}$ verifies $\text{sr}(\mathbf{\Xi}) < 1$, then $\mathbf{X}_k$ generated by \alg converges linearly to the optimal solution $\mathbf{X}^*$.
This condition can be verified by a suitable choice of parameters $\beta, \tau, r, \rho, \gamma, \eta, p$.


\bibliographystyle{IEEEtran}
\bibliography{references}
\end{document}